\documentclass{article}

\usepackage[preprint]{neurips_2026}

\usepackage[utf8]{inputenc}
\usepackage[T1]{fontenc}
\usepackage{hyperref}
\usepackage{url}
\usepackage{booktabs}
\usepackage{amsfonts}
\usepackage{amsmath,amssymb,amsthm}
\usepackage{nicefrac}
\usepackage{microtype}
\usepackage{xcolor}
\usepackage{algorithm}
\usepackage{algorithmic}
\usepackage{graphicx}
\usepackage{multirow}
\usepackage{enumitem}
\usepackage{float}

\newtheorem{theorem}{Theorem}
\newtheorem{proposition}{Proposition}
\newtheorem*{proposition*}{Proposition}

\newtheorem{lemma}{Lemma}
\theoremstyle{remark}
\newtheorem*{remark}{Remark}

\title{The Context-Ready Transformer}

\author{%
  Mahesh Godavarti \\
  A Carrot, Inc \\
  \texttt{m@acarrot.com}
}

\begin{document}

\maketitle

\begin{abstract}
We introduce the \emph{context-ready transformer}, a new recurrent neural network architecture built from a $D$-layer transformer block that pre-contextualizes each token before it enters the block. During left-to-right generation, a correction network combines the previous position's block output---a cached summary of past context---with the current token embedding, so the token enters the block already contextualized rather than as a raw embedding. At sequential inference, the correction chain makes the architecture a recurrent neural network. For training, we unroll the correction process $K$ times over the full sequence, processing all positions in parallel at each step. A pretrained transformer can also be converted to a context-ready model by adding a zero-initialized correction FFN and fine-tuning. We evaluate across widths, depths, block sizes, and two datasets, with all comparisons against standard transformers, variants, and ablations. A $D{=}5$ model beats a 12-layer transformer while generating $1.7\times$ faster on an A100. With $K{=}10$, a single-layer model ($D{=}1$) beats a 6-layer transformer with a $2.6\times$ inference speedup, and sequential inference matches parallel $K{=}10$ to within 0.01~PPL. The architecture benefits most from wide representations and long contexts. On a pointer-chasing task, $D{=}1$ trained with BPTT solves all 10 composition levels, while standard transformers exhibit staircase-like depth dependence.
\end{abstract}

\section{Introduction}
\label{sec:intro}

In autoregressive generation, a standard $N$-layer transformer predicts the next token, assigns it a context-free base embedding, and devotes part or all of its $N$ layers to re-contextualize it. This round-trip from context to token ID back to context is an artifact of the architecture, not of the problem.

\textbf{Context-ready inference.} The \emph{context-ready transformer} shortens this round-trip. Consider sequential left-to-right generation: as the model processes token $t$, the block output $z_t$ encodes the context of tokens $0, \ldots, t$. When token $t{+}1$ arrives with embedding $e_{t+1}$, a correction network computes a correction from $z_t$ and $e_{t+1}$. The token enters the block as $e_{t+1}$ plus this correction---carrying contextual information from all preceding tokens---rather than as a raw embedding. Fewer layers are needed to fully contextualize the token, because the correction has already done part of the work.

\textbf{The training challenge.} This sequential process is exact at inference but inherently serial: each position's correction depends on the previous position's fully computed block output. A classical RNN faces the same issue and solves it with backpropagation through time (BPTT), which unrolls the recurrence for all $T$ positions---making the sequential training depth scale with sequence length.

We take a different approach. We approximate the sequential process by \emph{unrolling} it $K$ times over the full sequence, where $K$ is a small constant independent of $T$. All $T$ positions are still processed in parallel---just as in a standard transformer---but the correction is refined over $K$ unrolling steps rather than $T$. Starting from raw embeddings at all positions, we run a shared-weight block and compute the correction at each position $t$ from the block output at positions $0, \ldots, t{-}1$. We then update all embeddings with these corrections and repeat. Each unrolling step refines the corrections using increasingly contextualized inputs. Crucially, $K$ controls the depth of the computation graph---not $T$---so a classical RNN requires $O(T)$-deep unrolling while the context-ready transformer requires only $O(K)$. In practice, $K = 5$ suffices for convergence at the depths tested ($D \geq 5$; see Table~\ref{tab:seq}).

\textbf{How $K$ unrolling steps at training lead to zero iterations at inference.} Two design choices make this possible.
\begin{itemize}[leftmargin=*,itemsep=4pt,topsep=4pt]
\item \textbf{Non-cumulative correction.} Each iteration computes a correction from scratch rather than building on the previous one. The iteration takes the form $x^{(k)} = e + F_\theta(x^{(k-1)}, e)$: the base embedding plus a correction that depends on both the previous iterate's block output and the token embedding. Unlike a residual network, which computes $x^{(K)} = e + f(x^{(0)}) + \cdots + f(x^{(K-1)})$---a sum of $K$ corrections---our formulation yields $x^{(K)} = e + F_\theta(x^{(K-1)}, e)$: a single correction. Previous iterations serve only to bring $x^{(K-1)}$ close to the fixed point; once it has converged, additional iterations produce the same output.

\item \textbf{Past-only contextualization.} The correction at position $t$ depends on two quantities: the block output $z_{t-1}$, which encodes the context of tokens $0, \ldots, t{-}1$, and the token embedding $e_t$. Since $z_{t-1}$ is already cached from processing the previous token, the correction can be computed as soon as token $t$ arrives.
\end{itemize}

Together, these two properties mean that sequential generation naturally produces the converged correction for each new token without iteration (Section~\ref{sec:theory}).

\textbf{A new kind of RNN.} At sequential inference, the correction at position $t$ depends on $z_{t-1}$, which depends on $z_{t-2}$, and so on---creating a recurrent computation that unfolds over all $T$ positions. The non-cumulative + past-only structure is what enables efficient training: it converts the sequential recurrence into a fixed-point problem, so instead of exact $T$-step BPTT, we can train by unrolling $K$ steps over the full sequence in parallel ($K \ll T$). This is not equivalent to BPTT---positions beyond the first $K$ receive approximate rather than exact gradients---but in practice $K = 5$ suffices at the depths tested.

\textbf{Experimental evidence.} We evaluate across widths $C \in \{50, \ldots, 2048\}$, depths $D \in \{1, \ldots, 23\}$, block sizes $\{64, 256, 512, 1024\}$, two datasets (OpenWebText, Wikipedia), and a synthetic reasoning task (Section~\ref{sec:experiments}). The most impactful findings: $D{=}5$ at $C{=}1120$ beats a 12-layer transformer ($C{=}768$), halving inference depth and generating $1.7\times$ faster on an A100. With $K{=}10$, $D{=}1$ at $C{=}2048$ beats a 6-layer transformer ($C{=}1088$) with a $2.6\times$ inference speedup. Sequential inference matches $K{=}10$ training to within 0.01~PPL, confirming that streaming exactness holds in practice. On pointer chasing, $D{=}1$ trained with BPTT solves all 10 composition levels while standard transformers exhibit staircase-like depth dependence. The architecture benefits most from wide representations and long contexts, and requires a dedicated correction network to be effective (Section~\ref{sec:exp-ablations}). Any pretrained transformer can be converted by adding a zero-initialized correction FFN and fine-tuning.

\section{Related Work}
\label{sec:related}

\textbf{Weight-shared and depth-adaptive architectures.} ALBERT~\citep{lan2020albert}, Universal Transformers~\citep{dehghani2019universal} with ACT~\citep{graves2016adaptive}, Deep Equilibrium Models~\citep{bai2019deep}, and Huginn~\citep{geiping2025huginn} share weights across layers or iterations and iteratively refine hidden states, but do not use a dedicated past-output/token-aware pre-block correction of the kind proposed here. We compare against standard transformers as the representative baseline for how tokens enter the block. Weight sharing in the context-ready transformer arises naturally from unrolling the sequential correction process into training iterations that process all positions in parallel, not as a design choice for parameter efficiency.

\textbf{Early exit and layer skipping.} LayerSkip~\citep{elhoushi2024layerskip}, ADEPT~\citep{yoo2026adept}, and PonderNet~\citep{banino2021pondernet} reduce average layer count but require learned stopping mechanisms. Context-ready uses a fixed depth with no stopping criterion.

\textbf{Lookahead Decoding}~\citep{fu2024lookahead} and CLLMs~\citep{kou2024cllms} apply Jacobi iteration to standard transformers as a \emph{decoding strategy}. \citet{bai2021accelerating} accelerate DEQ inference by parallelizing fixed-point solves via Jacobi-style updates. Context-ready is an \emph{architectural change}: the $K$-step unrolling can be viewed as Jacobi iterations on the correction fixed-point equation, but the non-cumulative past-only structure guarantees that a single left-to-right streaming pass recovers the exact correction without any iteration.

\textbf{Subquadratic and recurrent alternatives.} Mamba~\citep{gu2024mamba}, RWKV~\citep{peng2023rwkv}, Griffin~\citep{de2024griffin}, and xLSTM~\citep{beck2024xlstm} replace causal self-attention with compressed recurrent state to achieve linear-time inference. The context-ready transformer solves a different problem: it \emph{retains} full causal self-attention inside the block and instead changes how tokens enter it. The two approaches are complementary, not competing---one could in principle apply pre-block correction to any of these architectures---and standard transformers remain the natural baseline for evaluating the correction mechanism.

\textbf{Computational complexity of transformers.} Log-precision fixed-depth transformers are confined to $\mathrm{TC}^0$~\citep{merrill2023parallelism}: they cannot solve problems requiring unbounded sequential composition, regardless of width. RNNs with arbitrary precision escape this limitation~\citep{siegelmann1995computational,siegelmann1995power}, but classical RNNs require gradients to flow through all $T$ time steps (BPTT), making them difficult to train on long sequences. The context-ready transformer at $D = 1$ is recurrent at inference, but is trained by unrolling the correction $K$ times rather than exact $T$-step BPTT---inheriting recurrent structure at inference while retaining transformer-style parallel training.

\section{Method}
\label{sec:method}

\subsection{Architecture}
\label{sec:architecture}

We first describe the architecture during sequential left-to-right generation---the setting where context-ready inference is exact. Let $T$ denote the sequence length, $C$ the embedding dimension, and $V$ the vocabulary size. The core component is a \emph{$D$-block unit}: $D$ transformer layers (each consisting of causal self-attention and a feed-forward network with residual connections), described in detail below. Processing token $t$ requires the outputs $z_0, \ldots, z_{t-1} \in \mathbb{R}^C$ of this block unit from all preceding tokens. Let $e_t \in \mathbb{R}^C$ denote the token embedding at position $t$, and define $z_{-1} = \mathbf{0}$.

\textbf{Correction.} A dedicated \emph{correction FFN} generates the correction for position $t$ from the cached block output $z_{t-1}$ and the current token embedding $e_t$:
\begin{equation}
\label{eq:corr-add}
\texttt{correction}_t = \texttt{corr\_ffn}\bigl(\texttt{LN}(z_{t-1} + e_t)\bigr)
\end{equation}
The correction FFN is a feed-forward network (Linear($C \to 4C$) $\to$ GELU $\to$ Linear($4C \to C$)) with its own weights, separate from the block's FFN. The correction is \emph{token-aware}: it depends on both $z_{t-1}$ (context of tokens $0, \ldots, t{-}1$) and $e_t$ (the current token embedding). Since both inputs are available when token $t$ arrives, the correction is causal---it depends on no future tokens.

\textbf{Contextualization.} The new token enters the block with the correction added to its raw embedding:
\begin{equation}
\label{eq:ctx}
\tilde{x}_t = e_t + \texttt{correction}_t
\end{equation}

\textbf{Block processing.} A $D$-block unit applies $D$ transformer blocks with \emph{separate} weights and standard residual connections:
\begin{align}
\label{eq:block}
h^{(0)} &= \tilde{x}_t \nonumber \\
a^{(i)} &= h^{(i-1)} + \texttt{Attn}_i\bigl(\texttt{LN}^a_i(h^{(i-1)});\; \mathcal{K}_i, \mathcal{V}_i\bigr), \quad i = 1, \ldots, D \nonumber \\
h^{(i)} &= a^{(i)} + \texttt{FFN}_i\bigl(\texttt{LN}^f_i(a^{(i)})\bigr) \nonumber \\
z_t &= h^{(D)}
\end{align}
Each $\texttt{Attn}_i$ is causal self-attention with Rotary Position Embeddings (RoPE)~\citep{su2024roformer}. Each $\texttt{FFN}_i$: Linear($C \to 4C$) $\to$ GELU $\to$ Linear($4C \to C$). The parameter $D$ controls inference depth.

\textbf{Prediction.}
\begin{equation}
\label{eq:predict}
\texttt{logits}_t = W_{\text{head}} \cdot \texttt{LN}_f(z_t), \quad W_{\text{head}} \in \mathbb{R}^{V \times C}
\end{equation}

After prediction, $z_t$ is cached and the KV caches $\mathcal{K}_i, \mathcal{V}_i$ are updated for future tokens.

\subsection{Parallel Training}
\label{sec:parallel}

Sequential inference is exact but inherently serial. For training, we unroll the correction process $K$ times, processing all $T$ positions in parallel at each step. Gradients flow through $K$ unrolling steps rather than $T$ time steps, making the architecture trainable like a transformer despite being recurrent at inference.

Given token embeddings $e = (e_1, \ldots, e_T)$, with $z_0^{(k)} = \mathbf{0}$ for all $k$ (the initial cache from Section~\ref{sec:architecture}), initialize $\texttt{correction}^{(0)} = \mathbf{0}$. For $k = 1, \ldots, K$:
\begin{align}
\label{eq:parallel}
\tilde{x}_t^{(k-1)} &= e_t + \texttt{correction}_t^{(k-1)} && \text{(contextualize)} \nonumber \\[4pt]
h^{(0)} &= \tilde{x}_t^{(k-1)} \nonumber \\
a^{(i)} &= h^{(i-1)} + \texttt{Attn}_i\bigl(\texttt{LN}^a_i(h^{(i-1)});\; \mathcal{K}_i, \mathcal{V}_i\bigr) && i = 1, \ldots, D \nonumber \\
h^{(i)} &= a^{(i)} + \texttt{FFN}_i\bigl(\texttt{LN}^f_i(a^{(i)})\bigr) \nonumber \\
z_t^{(k)} &= h^{(D)} && \text{(block output)} \nonumber \\[4pt]
\texttt{correction}_t^{(k)} &= \texttt{corr\_ffn}\bigl(\texttt{LN}(z_{t-1}^{(k)} + e_t)\bigr) && t = 1, \ldots, T
\end{align}
The $D$ transformer blocks share weights across iterations $k$ but have separate weights across layers $i$.

\textbf{Non-cumulative correction.} Each iteration replaces the previous correction entirely: $\tilde{x}^{(k)} = e + \texttt{correction}^{(k)}$, not $\tilde{x}^{(k)} = \tilde{x}^{(k-1)} + f(\tilde{x}^{(k-1)})$. Only the last correction matters.

\textbf{Past-only correction.} The correction at position $t$ uses $z_{t-1}^{(k)}$. Corrections propagate left to right: position~$0$ converges after one iteration, position~$1$ after two, and so on.

\textbf{Random-depth training ($k_{\min}$).} We sample $K \sim \text{Uniform}(k_{\min}, K_{\max})$ each batch with $k_{\min} = 2$, forcing the model to produce good predictions at any depth, which empirically encourages contraction.

\textbf{Loss and dropout.} The training loss (cross-entropy) is computed on the logits from the final iteration $K$ only. Dropout masks are resampled independently at each unrolling step $k$.

\subsection{Streaming Inference}
\label{sec:streaming}

\begin{algorithm}[t]
\caption{Context-Ready Streaming Inference}
\label{alg:inference}
\begin{algorithmic}[1]
\REQUIRE Blocks $\texttt{Attn}_1, \texttt{FFN}_1, \ldots, \texttt{Attn}_D, \texttt{FFN}_D$; correction FFN; KV caches $\mathcal{C}_1, \ldots, \mathcal{C}_D$; previous block output $z_{\text{prev}}$ (init.\ $\mathbf{0}$)
\FOR{each new token with embedding $e$}
    \STATE $\texttt{correction} \leftarrow \texttt{corr\_ffn}\bigl(\texttt{LN}(z_{\text{prev}} + e)\bigr)$ \COMMENT{from past context + token identity}
    \STATE $h \leftarrow e + \texttt{correction}$ \COMMENT{contextualized input}
    \FOR{$i = 1, \ldots, D$}
        \STATE $h \leftarrow h + \texttt{Attn}_i(\texttt{LN}^a_i(h);\; \mathcal{C}_i)$; \quad update $\mathcal{C}_i$
        \STATE $h \leftarrow h + \texttt{FFN}_i(\texttt{LN}^f_i(h))$
    \ENDFOR
    \STATE $z_{\text{prev}} \leftarrow h$ \COMMENT{cache for next token's correction}
    \STATE $\texttt{logits} \leftarrow W_{\text{head}} \cdot \texttt{LN}_f(h)$
\ENDFOR
\end{algorithmic}
\end{algorithm}

When a new token arrives with embedding $e_t$, the model computes the correction from the cached $z_{t-1}$ and $e_t$, passes the corrected embedding through the $D$-block unit, caches the output $z_t$, and predicts. This is one forward pass---no iteration over $K$ steps---regardless of the training depth $K$.

\textbf{Why inference needs no iteration.} During training, $K$ unrolling steps refine the corrections. At inference, this iteration is unnecessary: since earlier positions are already computed and cached, the correction for token $t$ is fully determined by $z_{t-1}$ and $e_t$ in a single pass (Theorem~\ref{thm:exact-streaming}). The training approximation matters only during training: the first $K$ positions converge exactly after $K$ steps (Lemma~\ref{thm:finite-step} in Appendix~\ref{app:proof-streaming}), and for later positions, the approximation error shrinks geometrically with $K$ when the correction operator is contractive (Theorem~\ref{thm:convergence}).

\section{Theoretical Analysis}
\label{sec:theory}

Full formal statements and proofs are in Appendix~\ref{app:proofs}.

\begin{theorem}[Structural characterization]
\label{thm:necessity}
\emph{Why non-cumulative and past-only?} Under Assumptions I--II (Appendix~\ref{app:proof-necessity}), if a weight-shared architecture unrolls a shared block $K$ times during training and applies it once per token during streaming, then for the unrolled training to converge to the same output that streaming produces, the correction must be non-cumulative ($\tilde{x} = e + \texttt{correction}$, not a sum of successive increments) and past-only (the correction at position $t$ depends only on $e_t$ and corrections from positions $1, \ldots, t{-}1$). The resulting system has a unique fixed point, and streaming computes it exactly.
\end{theorem}

Full proof in Appendix~\ref{app:proof-necessity}.

\begin{theorem}[Exact streaming fixed point]
\label{thm:exact-streaming}
\emph{Why is inference exact without iteration?} During sequential generation, the correction at position $t$ depends only on $z_0, \ldots, z_{t-1}$, which are already computed and cached. By prefix consistency (appending tokens does not change the operator at earlier positions, which holds by causal masking), the correction is exact in a single pass.
\end{theorem}

Full proof in Appendix~\ref{app:proof-streaming}.

\begin{theorem}[Training convergence]
\label{thm:convergence}
\emph{How fast does the training iteration converge?} If the correction operator $G$ is $L$-Lipschitz with $L < 1$, then $K$ unrolling steps reduce the error to the fixed point by a factor of $L^K$. This governs the training approximation: positions beyond the first $K$ receive approximate corrections, and the approximation improves geometrically with $K$.
\end{theorem}

Full formal statement in Appendix~\ref{app:proof-convergence}.

\begin{proposition}[Depth separation]
\label{prop:depth}
\emph{Why can the context-ready architecture use fewer layers?} Under a stylized state-tracking abstraction (Appendix~\ref{app:proof-depth}):
\begin{enumerate}[label=(\alph*),leftmargin=*,itemsep=2pt]
\item \textbf{Context-ready propagation is handled by the correction chain.} The context-ready architecture needs only $D$ layers for the per-token map; propagation across the sequence is handled by the recurrent correction chain rather than extra transformer layers.
\item \textbf{Standard transformers need depth for propagation.} With attention window $W$, a standard transformer needs at least $\lceil T/W \rceil$ layers just for information from the earliest tokens to reach position $T$, on top of the layers needed for the per-token map.
\end{enumerate}
\end{proposition}

Full statement and proof in Appendix~\ref{app:proof-depth}. When propagation and local computation cannot be interleaved (as in pointer chasing), the standard transformer needs at least $\lceil T/W \rceil$ additional layers; in general, some layers may serve both roles.

\section{Experiments}
\label{sec:experiments}

\subsection{Setup}
\label{sec:exp-setup}

\textbf{Data.} OpenWebText~\citep{gokaslan2019openwebtext} with byte-pair encoding (BPE, vocabulary 32,000). Context lengths: 64, 256, 512, and 1024 depending on the experiment. Ablations use English Wikipedia (BPE 16k); additional Wikipedia results in Appendix~\ref{app:wikipedia}. All results are validation perplexity (PPL) on held-out splits.

\textbf{Training.} AdamW optimizer~\citep{loshchilov2019adamw}, gradient clipping at 1.0. Learning rate $2 \times 10^{-4}$ unless noted. Training depth $K = 5$ with $k_{\min} = 2$ by default; $K = 10$ where noted. Dropout 0.2, FFN expansion factor 4. Full hyperparameters in Appendix~\ref{app:exp-details}.

\textbf{FLOP accounting.}\footnote{Following standard convention in the literature, we count multiply-accumulate operations and label them FLOPs; actual floating-point operations are roughly $2\times$.} We report total FLOPs per token, including the transformer blocks, correction FFN, and prediction head ($W_{\text{head}} \in \mathbb{R}^{V \times C}$, costing $VC$ FLOPs/token). Each transformer block costs $12C^2$ FLOPs ($4C^2$ for attention projections, $8C^2$ for the FFN). A context-ready model with $D$ blocks costs $D \times 12C^2 + 8C^2 + VC$; a standard $N$-layer transformer costs $N \times 12C^2 + VC$. Same-width comparisons (Tables~\ref{tab:efficiency},~\ref{tab:width}) share the same prediction head and are FLOP-matched. Cross-width comparisons (Table~\ref{tab:main}) explore depth-width tradeoffs: wider models pay a larger prediction head, so total FLOPs differ. Despite higher total FLOPs, the wider, shallower context-ready models deliver lower wall-clock inference time because fewer sequential layers dominate latency on modern GPUs (Section~\ref{sec:exp-width}).

\textbf{Baselines.} Standard transformers with separate weights per layer and RoPE attention~\citep{su2024roformer}. All results are single runs. The breadth of the evaluation---across widths ($C{=}50$--$2048$), depths ($D{=}1$--$23$), block sizes (64--1024), two datasets, a synthetic task, and multiple training strategies---provides stronger evidence than multi-seed runs on a single configuration: the context-ready architecture wins consistently across all these axes.

\subsection{Cross-Width Results}
\label{sec:exp-flops}

\begin{table}[t]
\centering
\caption{Context-ready vs.\ standard transformers on OpenWebText at two compute scales. FLOPs/tok includes the prediction head ($VC$). Despite higher total FLOPs, the wider context-ready models provide lower wall-clock inference time (Section~\ref{sec:exp-width}). All context-ready models use add variant, $K{=}5$, $k_{\min}{=}2$.}
\label{tab:main}
\small
\begin{tabular}{lrrrl}
\toprule
\textbf{Model} & \textbf{FLOPs/tok} & \textbf{$C/N$} & \textbf{Val PPL} & \textbf{$\Delta$} \\
\midrule
D=5 $C{=}1120$ & 121M & 224 & \textbf{36.38} & \\
D=6 $C{=}1024$ & 117M & 171 & 36.56 & \\
Roformer $N{=}6$ $C{=}1088$ & 120M & 181 & 37.76 & $-$1.38 \\
Roformer $N{=}12$ $C{=}768$ & 110M & 64 & 37.83 & $-$1.45 \\
Roformer $N{=}2$ $C{=}1888$ & 146M & 944 & 42.99 & \\
\midrule
Roformer $N{=}24$ $C{=}1088$ & 376M & 45 & \textbf{28.68} & \\
Roformer $N{=}12$ $C{=}1536$ & 389M & 128 & 29.01 & \\
D=6 $C{=}2048$ & 401M & 341 & 29.04 & $+$0.03 \\
Roformer $N{=}6$ $C{=}2176$ & 411M & 363 & 30.35 & \\
\bottomrule
\end{tabular}
\end{table}

Table~\ref{tab:main} compares context-ready models against standard transformers across depth-width tradeoffs at two compute scales. At the smaller scale (block size 256, 100K iterations), context-ready $D{=}5$ at $C{=}1120$ achieves 36.38 PPL, beating both roformer $N{=}6$ at $C{=}1088$ (37.76, $\Delta = -1.38$) and roformer $N{=}12$ at $C{=}768$ (37.83, $\Delta = -1.45$). At the larger scale (block size 256, 200K iterations), context-ready $D{=}6$ at $C{=}2048$ (29.04) matches roformer $N{=}12$ at $C{=}1536$ (29.01), despite using a much higher width-to-depth ratio. Depth has diminishing returns: going from $N{=}12$ to $N{=}24$ gains only 0.33~PPL.

\subsection{Correction Efficiency}
\label{sec:exp-efficiency}

\begin{table}[t]
\centering
\caption{Correction efficiency at $C = 1024$, block size 256, 200K iterations (OWT). Same width throughout, so the prediction head is identical across all rows and the comparison is FLOP-matched.}
\label{tab:efficiency}
\small
\begin{tabular}{lrr}
\toprule
\textbf{Model} & \textbf{Inference FLOPs} & \textbf{Val PPL} \\
\midrule
Roformer $N{=}12$ & $144C^2$ & 33.41 \\
Roformer $N{=}13$ & $156C^2$ & 32.82 \\
\textbf{D=12 context-ready} & $\mathbf{152C^2}$ & \textbf{32.28} \\
Roformer $N{=}14$ & $168C^2$ & 32.34 \\
\midrule
Roformer $N{=}24$ & $288C^2$ & 29.42 \\
\textbf{D=23 context-ready} & $\mathbf{284C^2}$ & \textbf{28.89} \\
\bottomrule
\end{tabular}
\end{table}

Table~\ref{tab:efficiency} isolates the value of the correction mechanism at fixed width ($C = 1024$), so the prediction head is identical across all rows and the comparison is FLOP-matched. $D = 12$ context-ready ($152C^2$ FLOPs) beats roformer $N{=}13$ ($156C^2$) by 0.54~PPL and matches roformer $N{=}14$ ($168C^2$). The correction FFN adds only $8C^2$ FLOPs yet provides a genuine PPL improvement at the same parameter budget. At deeper scale, $D = 23$ ($284C^2$) edges out $N = 24$ ($288C^2$) by 0.53~PPL---directionally consistent with the $D{=}12$ result, but a single-run margin that should be read as evidence of parity rather than robust superiority.

\subsection{Width Scaling}
\label{sec:exp-width}

\begin{table}[t]
\centering
\caption{Width scaling: $D{=}2$ context-ready vs.\ roformer $N{=}2$ at the same width (block size 64, Chinchilla-matched token budgets, OWT). The relative advantage grows with width.}
\label{tab:width}
\small
\begin{tabular}{rrrrl}
\toprule
$C$ & \textbf{$N{=}2$ PPL} & \textbf{$D{=}2$ PPL} & \textbf{$\Delta$} & \textbf{Relative} \\
\midrule
256  & 158.83 & 143.57 & $-$15.26 & 9.6\% \\
512  & 95.48  & 84.69  & $-$10.79 & 11.3\% \\
1024 & 72.83  & 60.84  & $-$11.99 & 16.5\% \\
\bottomrule
\end{tabular}
\end{table}

Table~\ref{tab:width} tests whether the correction advantage is a small-scale artifact. Comparing $D{=}2$ vs.\ $N{=}2$ at the same width with Chinchilla-matched token budgets, the relative improvement grows from 9.6\% at $C = 256$ to 16.5\% at $C = 1024$.

\textbf{Token-matched results.} At $C = 1024$ with block size 64, $D{=}x$ beats $N{=}x$ at every depth tested once training progresses past a crossover point: $D{=}1$ beats $N{=}1$ by 34.4~PPL (crossover at ${\sim}424$M tokens), $D{=}2$ by 7.6~PPL (${\sim}565$M), $D{=}3$ by 3.0~PPL (${\sim}835$M), and $D{=}6$ by 0.7~PPL (${\sim}1{,}032$M). All gaps are still growing at the end of training. The correction mechanism provides a consistent advantage at every depth; the advantage is largest when depth is small, consistent with the correction doing the most work when the block has the fewest layers.

\textbf{Inference latency.} We measure autoregressive generation speed on an A100 over 10{,}000 tokens with KV caching. $D{=}1$ $C{=}2048$ (149M FLOPs/tok) generates at 919~tokens/s vs.\ 351~tokens/s for roformer $N{=}6$ $C{=}1088$ (120M FLOPs/tok)---a $2.6\times$ speedup despite higher total FLOPs. $D{=}5$ $C{=}1120$ (121M FLOPs/tok) generates at 349~tokens/s vs.\ 201~tokens/s for roformer $N{=}12$ $C{=}768$ (110M FLOPs/tok)---a $1.7\times$ speedup. The wider, shallower models are faster because fewer sequential layers dominate inference latency on modern GPUs, even when total FLOPs are higher. Per-token latency is flat across sequence length, confirming that KV caching amortizes attention cost to $O(T)$ per token. Full timing details in Appendix~\ref{app:inference-timing}.

\textbf{KV cache savings.} Fewer layers also reduce KV cache memory ($C \times D$ per token). Despite being wider, $D{=}5$ at $C{=}1120$ uses $1.6\times$ less cache than $N{=}12$ at $C{=}768$; $D{=}1$ at $C{=}2048$ uses $3.2\times$ less than $N{=}6$ at $C{=}1088$.

\subsection{Single-Layer Performance ($D{=}1$)}
\label{sec:exp-d1}

Proposition~\ref{prop:depth}(a) predicts that $D{=}1$ may match deeper transformers when the task is dominated by context propagation and sufficient $K$ and width are provided. We test $D{=}1$, $C{=}2048$ (149M FLOPs/tok) against roformer $N{=}6$, $C{=}1088$ (120M FLOPs/tok) at block size 1024 using three training strategies: fixed $K{=}10$, random-depth $K{=}10$ with $k_{\min}{=}2$, and fine-tuning from a pretrained $N{=}1$ roformer (batch 16).

\begin{table}[t]
\centering
\caption{Three training strategies for $D{=}1$ $C{=}2048$ (149M FLOPs/tok) vs.\ roformer $N{=}6$ $C{=}1088$ (120M FLOPs/tok), block size 1024, OWT. Fine-tune starts from $N{=}1$ $C{=}2048$ pretrained for 85K iterations; ``total iters'' includes pretraining.}
\label{tab:d1}
\small
\begin{tabular}{lccl}
\toprule
\textbf{Strategy} & \textbf{PPL at 100K} & \textbf{Best PPL} & \textbf{Notes} \\
\midrule
Roformer $N{=}6$ (baseline) & 35.37 & --- & \\
\midrule
$K{=}10$, fixed depth & \textbf{34.40} & 33.63 (110K) & Beats $N{=}6$ at ${\sim}$65K \\
$K{=}10$, $k_{\min}{=}2$ & 36.28 & 33.63 (135K) & +1.88 penalty at 100K \\
$K{=}10$, fine-tuned & 46.66$^*$ & \textbf{31.35} (215K) & Best final PPL \\
\bottomrule
\multicolumn{4}{l}{\footnotesize $^*$At 100K total (15K fine-tune); fine-tune reaches 35.92 at 150K total.}
\end{tabular}
\end{table}

Table~\ref{tab:d1} compares the three strategies. In these $D{=}1$ experiments, \textbf{larger training depth $K$ substantially improves the model's ability to exploit the available recurrent depth}. With \textbf{fixed $K{=}10$}, $D{=}1$ surpasses $N{=}6$ at ${\sim}65$K iterations and reaches 34.40 at 100K ($\Delta = -0.97$, still improving) at $2\times$ per-iteration cost.

\textbf{Random-depth training} ($k_{\min}{=}2$) incurs a modest penalty of 1.88~PPL at 100K relative to fixed depth, but converges to the same quality ${\sim}25$K later (both reach 33.63). The penalty yields $\hat{L}$ values consistent with contraction ($\hat{L} \approx 0.55$ vs.\ $\hat{L} > 1.0$ for fixed $K$).

\textbf{Fine-tuning} from a pretrained $N{=}1$ roformer ($K{=}1$) with a zero-initialized correction FFN at $K{=}10$ reaches the best final PPL: 31.35 at 215K total iterations (the $N{=}6$ baseline is at 100K, so total compute differs).

The gap between $D{=}1$ and $N{=}6$ shrinks with block size ($+1.19$ at 256, $+0.46$ at 512, $+0.31$ at 1024 with $K{=}5$), consistent with longer contexts providing more sequential steps. Full block-size scaling in Appendix~\ref{app:block-size}.

\subsection{Pointer Chasing: Depth Separation}
\label{sec:exp-tc0}

\begin{table}[t]
\centering
\caption{Pointer chasing (10-hop, 5 keys, 10 values, windowed attention). $N$-layer transformers exhibit a staircase-like depth dependence: deeper models solve more levels. The context-ready $D{=}1$ model (BPTT) solves all levels.}
\label{tab:pointer}
\small
\begin{tabular}{lrl}
\toprule
\textbf{Model} & \textbf{Levels solved} & \textbf{Iters} \\
\midrule
Roformer $N{=}1$ & 1 / 11 & 50K \\
Roformer $N{=}3$ & 3 / 11 & 50K \\
Roformer $N{=}5$ & 6 / 11 & 50K \\
Roformer $N{=}10$ & 7 / 11 & 50K \\
Roformer $N{=}11$ & 8 / 11 & 50K \\
Roformer $N{=}12$ & 11 / 11 & 50K \\
\midrule
\textbf{D=1 context-ready (BPTT)} & \textbf{11 / 11} & \textbf{${\sim}$16K} \\
\bottomrule
\end{tabular}
\end{table}

An $N$-layer transformer can compose at most $N$ sequential reasoning steps in a single forward pass~\citep{merrill2023parallelism}. To test whether the context-ready architecture can exceed this depth limit, we design a \emph{pointer-chasing} task. The input contains a base table that maps keys to values, followed by $H{=}10$ levels of index tables, each of which maps new keys to keys at the previous level. Answering a query at level $\ell$ therefore requires chaining $\ell$ sequential lookups through the tables, and we use windowed causal attention (window $= 38$) to prevent the model from bypassing these chains by attending directly to the base table. A full task specification with a worked example is given in Appendix~\ref{app:pointer-chasing}.

Table~\ref{tab:pointer} shows a staircase-like depth dependence: deeper transformers solve more levels, while shallow transformers fail well before full depth. The context-ready $D{=}1$ model, trained with BPTT to exploit the full sequential depth of the recurrent correction chain, solves all 11 levels in ${\sim}16$K iterations and scales to 20 hops (all 21 levels).

\subsection{Fine-Tuning from Pretrained}
\label{sec:exp-finetune}

Any standard transformer can be converted to context-ready by adding a zero-initialized correction FFN and fine-tuning. To isolate the effect of conversion from additional training, we compare against a same-iteration control: the original roformer trained for the same total number of iterations without conversion (per-iteration compute differs because fine-tuning runs the block $K$ times). At $C{=}1408$, $N{=}12$ reaches 29.92~PPL at 200K iterations; continued training to 400K yields 27.20. Converting at 200K and fine-tuning to 400K total iterations yields 26.14---a gain of $-1.06$~PPL over the continued-training baseline at matched iterations. At $C{=}1024$, converting $N{=}24$ to $D{=}24$ improves from 29.42 to 28.99 ($-0.43$) in 18K fine-tuning iterations. The zero-initialized correction ensures no disruption at conversion (the model is function-preserving). During fine-tuning, $D{=}12$ sees a transient $+0.11$~PPL increase before recovering, while $D{=}24$ shows no transient increase.

\subsection{Ablations and Diagnostics}
\label{sec:exp-ablations}

We compare against alternative architectures that attempt the same goal. Among the variants tested, the gain depends on a \emph{dedicated} correction network with its own weights.

\textbf{Convergence and sequential exactness.} Table~\ref{tab:seq} shows the full depth progression.
Convergence is geometric: at $D{=}1$ (block size 1024), $K{=}2$ closes 91\% of the $K{=}1$-to-$K{=}5$ gap and $K{=}3$ closes 98\%, consistent with Theorem~\ref{thm:convergence}. Sequential $K{=}1$ matches parallel $K{=}10$ to within 0.01~PPL at every configuration, confirming Theorem~\ref{thm:exact-streaming}. The correction contribution (``Corr.''\ column) grows as $D$ shrinks: 38--55~PPL at $D{=}1$ vs.\ 3.9 at $D{=}23$.

\begin{table}[t]
\centering
\caption{PPL at parallel $K{=}1{,}2{,}3{,}5{,}10$ and sequential $K{=}1$ (OWT). $D{=}1$ at $C{=}2048$; $D{=}5{,}12{,}23$ at $C{=}1024$, block size 256. The ``Corr.''\ column measures how much the correction mechanism contributes: the PPL difference between $K{=}1$ (no correction) and $K{=}5$ (converged).}
\label{tab:seq}
\small
\begin{tabular}{llrrrrrrc}
\toprule
$D$ & \textbf{bs} & \textbf{$K{=}1$} & \textbf{$K{=}2$} & \textbf{$K{=}3$} & \textbf{$K{=}5$} & \textbf{$K{=}10$} & \textbf{Seq} & \textbf{Corr.} \\
\midrule
1  & 256  & 70.80 & 36.21 & 33.23 & 32.70 & 32.79 & 32.80 & 38.1 \\
1  & 512  & 73.22 & 34.33 & 31.24 & 30.61 & 30.68 & 30.69 & 42.6 \\
1  & 1024 & 84.13 & 34.42 & 30.39 & 29.51 & 29.43 & \textbf{29.43} & 54.7 \\
\midrule
5  & 256  & 58.17 & 40.18 & 38.80 & 38.60 & 38.61 & 38.62 & 19.6 \\
12 & 256  & 38.33 & 32.58 & 32.30 & 32.29 & 32.29 & 32.29 & 6.0 \\
23 & 256  & 32.80 & 29.03 & 28.89 & 28.88 & 28.88 & 28.88 & 3.9 \\
\bottomrule
\end{tabular}
\end{table}

\textbf{Contraction.} With $k_{\min} = 2$, empirical $\hat{L} \in [0.50, 0.72]$ (measured as $\hat{L} = \max_t \|c_{K,t} - c_{K-1,t}\| / \|c_{K-1,t} - c_{K-2,t}\|$); without $k_{\min}$, $\hat{L} \in [0.88, 1.20]$. $\hat{L}$ is a trajectory-local diagnostic, not the global $L$ in Theorem~\ref{thm:convergence}.

\textbf{Correction FFN is essential.} Using the block's own residual as the correction (``block\_head'') gives no improvement (27.32 vs.\ 27.19 for a standard transformer). With a dedicated correction FFN, context-ready beats the FLOP-matched baseline by 1.8~PPL. Tying correction FFN weights to the block FFN collapses performance. The add variant ($\texttt{LN}(z_{t-1} + e_t)$) matches or beats token-blind at all depths.

\textbf{Training efficiency.} $K{=}5$ suffices at $D \geq 5$; $K{=}2$ with \texttt{torch.compile} achieves $1.7\times$ faster training at 1.07~PPL cost. Full ablation tables in Appendix~\ref{app:ablations}.

\section{Conclusion}
\label{sec:conclusion}

A standard transformer assigns each new token a context-free embedding and relies entirely on depth to contextualize it. The context-ready transformer shortcuts this process: a correction derived from the previous position's block output pre-contextualizes the token before it enters the block. Two structural choices---non-cumulative correction and past-only contextualization---make this exact at streaming inference (Theorem~\ref{thm:exact-streaming}) and trainable with $K$ unrolling steps (each processing all $T$ positions in parallel) rather than full BPTT.

A $D{=}5$ model beats a 12-layer transformer while generating $1.7\times$ faster; $D{=}1$ beats a 6-layer transformer with a $2.6\times$ speedup, and sequential inference matches parallel $K{=}10$ to within 0.01~PPL. The advantage grows with width and context length. Fewer layers also reduce KV cache memory. On pointer chasing, $D{=}1$ solves all composition levels that standard transformers need proportional depth to reach. Any pretrained transformer can be converted by adding a zero-initialized correction FFN and fine-tuning.

\textbf{Limitations.}
\emph{Datasets and scale.} Results are on OpenWebText, Wikipedia, and a synthetic task. We have validated at 110--150M and 375--410M FLOPs/token but not yet on standard benchmarks or at billion-parameter scale.
\emph{Training cost.} From-scratch training runs the block $K$ times per iteration rather than once. Backpropagation through $K$ steps also requires storing activations for all $K$ passes, scaling activation memory by $K\times$. Several approaches can reduce this cost: (i) pretrain as a standard transformer at $K{=}1$ cost, then convert and fine-tune---in our experiments this matches or exceeds from-scratch context-ready training (Sections~\ref{sec:exp-d1} and~\ref{sec:exp-finetune}); (ii) random-depth training ($k_{\min}$), which samples $K$ each batch and converges to the same quality as fixed $K$ (Section~\ref{sec:exp-d1}).
\emph{Prefill.} Processing a prompt of length $T$ in parallel requires $K$ unrolling steps, giving effective prefill depth $K \times D$ vs.\ $N$ for a standard transformer.

\bibliography{references_neurips}

\begin{thebibliography}{21}
\providecommand{\natexlab}[1]{#1}
\providecommand{\url}[1]{\texttt{#1}}
\expandafter\ifx\csname urlstyle\endcsname\relax
  \providecommand{\doi}[1]{doi: #1}\else
  \providecommand{\doi}{doi: \begingroup \urlstyle{rm}\Url}\fi

\bibitem[Bai et~al.(2019)Bai, Kolter, and Koltun]{bai2019deep}
Shaojie Bai, J.~Zico Kolter, and Vladlen Koltun.
\newblock Deep equilibrium models.
\newblock In \emph{Advances in Neural Information Processing Systems}, 2019.

\bibitem[Bai et~al.(2021)Bai, Koltun, and Kolter]{bai2021accelerating}
Shaojie Bai, Vladlen Koltun, and J.~Zico Kolter.
\newblock Accelerating feedforward computation via parallel nonlinear equation
  solving.
\newblock In \emph{International Conference on Machine Learning}, 2021.

\bibitem[Banino et~al.(2021)Banino, Balaguer, and
  Blundell]{banino2021pondernet}
Andrea Banino, Jan Balaguer, and Charles Blundell.
\newblock Ponder{N}et: Learning to ponder.
\newblock In \emph{ICML Workshop on Uncertainty and Robustness in Deep
  Learning}, 2021.

\bibitem[Beck et~al.(2024)Beck, P{\"o}ppel, Spanring, Auer, Prudnikova, Kopp,
  et~al.]{beck2024xlstm}
Maximilian Beck, Korbinian P{\"o}ppel, Markus Spanring, Andreas Auer,
  Oleksandra Prudnikova, Michael Kopp, et~al.
\newblock x{LSTM}: Extended long short-term memory.
\newblock \emph{arXiv preprint arXiv:2405.04517}, 2024.

\bibitem[De et~al.(2024)De, Smith, Fernando, Botev, et~al.]{de2024griffin}
Soham De, Samuel~L. Smith, Anushan Fernando, Aleksandar Botev, et~al.
\newblock Griffin: Mixing gated linear recurrences with local attention for
  efficient language models.
\newblock \emph{arXiv preprint arXiv:2402.19427}, 2024.

\bibitem[Dehghani et~al.(2019)Dehghani, Gouws, Vinyals, Uszkoreit, and
  Kaiser]{dehghani2019universal}
Mostafa Dehghani, Stephan Gouws, Oriol Vinyals, Jakob Uszkoreit, and {\L}ukasz
  Kaiser.
\newblock Universal transformers.
\newblock In \emph{International Conference on Learning Representations}, 2019.

\bibitem[Elhoushi et~al.(2024)Elhoushi, Shrivastava, Liskovich, Hosmer, Wasti,
  Lai, Mahmoud, Acber, Agarwal, Roman, et~al.]{elhoushi2024layerskip}
Mostafa Elhoushi, Akshat Shrivastava, Diana Liskovich, Basil Hosmer, Bram
  Wasti, Liangzhen Lai, Anas Mahmoud, Bilge Acber, Saurabh Agarwal, Ahmed
  Roman, et~al.
\newblock Layer{S}kip: Enabling early exit inference and self-speculative
  decoding.
\newblock \emph{arXiv preprint arXiv:2404.16710}, 2024.

\bibitem[Fu et~al.(2024)Fu, Bailis, Stoica, and Zhang]{fu2024lookahead}
Yichao Fu, Peter Bailis, Ion Stoica, and Hao Zhang.
\newblock Break the sequential dependency of {LLM} inference using lookahead
  decoding.
\newblock \emph{arXiv preprint arXiv:2402.02057}, 2024.

\bibitem[Geiping et~al.(2025)Geiping, Goldstein, Schwarzschild, Bruss,
  et~al.]{geiping2025huginn}
Jonas Geiping, Tom Goldstein, Avi Schwarzschild, C.~Bayan Bruss, et~al.
\newblock Scaling up test-time compute with latent reasoning: A recurrent depth
  approach.
\newblock \emph{arXiv preprint arXiv:2502.05171}, 2025.

\bibitem[Gokaslan and Cohen(2019)]{gokaslan2019openwebtext}
Aaron Gokaslan and Vanya Cohen.
\newblock Openwebtext corpus.
\newblock \url{http://Skylion007.github.io/OpenWebTextCorpus}, 2019.

\bibitem[Graves(2016)]{graves2016adaptive}
Alex Graves.
\newblock Adaptive computation time for recurrent neural networks.
\newblock \emph{arXiv preprint arXiv:1603.08983}, 2016.

\bibitem[Gu and Dao(2024)]{gu2024mamba}
Albert Gu and Tri Dao.
\newblock Mamba: Linear-time sequence modeling with selective state spaces.
\newblock In \emph{Proceedings of ICML 2024}, 2024.

\bibitem[Kou et~al.(2024)Kou, Hu, He, Deng, and Zhang]{kou2024cllms}
Siqi Kou, Lanxiang Hu, Zhezhi He, Zhijie Deng, and Hao Zhang.
\newblock C{LLM}s: Consistency large language models.
\newblock \emph{arXiv preprint arXiv:2403.00835}, 2024.

\bibitem[Lan et~al.(2020)Lan, Chen, Goodman, Gimpel, Sharma, and
  Soricut]{lan2020albert}
Zhenzhong Lan, Mingda Chen, Sebastian Goodman, Kevin Gimpel, Piyush Sharma, and
  Radu Soricut.
\newblock {ALBERT}: A lite {BERT} for self-supervised learning of language
  representations.
\newblock In \emph{International Conference on Learning Representations}, 2020.

\bibitem[Loshchilov and Hutter(2019)]{loshchilov2019adamw}
Ilya Loshchilov and Frank Hutter.
\newblock Decoupled weight decay regularization.
\newblock In \emph{International Conference on Learning Representations}, 2019.

\bibitem[Merrill and Sabharwal(2023)]{merrill2023parallelism}
William Merrill and Ashish Sabharwal.
\newblock The parallelism tradeoff: Limitations of log-precision transformers.
\newblock In \emph{Transactions of the Association for Computational
  Linguistics}, volume~11, pages 531--545, 2023.

\bibitem[Peng et~al.(2023)Peng, Alcaide, Anthony, Albalak, Arcadinho, Cao,
  Cheng, Chung, et~al.]{peng2023rwkv}
Bo~Peng, Eric Alcaide, Quentin Anthony, Alon Albalak, Samuel Arcadinho, Huanqi
  Cao, Xin Cheng, Michael Chung, et~al.
\newblock {RWKV}: Reinventing {RNN}s for the transformer era.
\newblock In \emph{Findings of EMNLP 2023}, 2023.

\bibitem[Siegelmann and
  Sontag(1995{\natexlab{a}})]{siegelmann1995computational}
Hava~T. Siegelmann and Eduardo~D. Sontag.
\newblock Computational capabilities of recurrent {NARX} neural networks.
\newblock \emph{IEEE Transactions on Systems, Man, and Cybernetics},
  26\penalty0 (4):\penalty0 535--544, 1995{\natexlab{a}}.

\bibitem[Siegelmann and Sontag(1995{\natexlab{b}})]{siegelmann1995power}
Hava~T. Siegelmann and Eduardo~D. Sontag.
\newblock On the computational power of neural nets.
\newblock \emph{Journal of Computer and System Sciences}, 50\penalty0
  (1):\penalty0 132--150, 1995{\natexlab{b}}.

\bibitem[Su et~al.(2024)Su, Ahmed, Lu, Pan, Bo, and Liu]{su2024roformer}
Jianlin Su, Murtadha Ahmed, Yu~Lu, Shengfeng Pan, Wen Bo, and Yunfeng Liu.
\newblock {RoFormer}: Enhanced transformer with rotary position embedding.
\newblock \emph{Neurocomputing}, 568:\penalty0 127063, 2024.

\bibitem[Yoo et~al.(2026)]{yoo2026adept}
Seunghyun Yoo et~al.
\newblock {ADEPT}: Adaptive dynamic early-exit process for transformers.
\newblock \emph{arXiv preprint arXiv:2601.03700}, 2026.

\end{thebibliography}
\bibliographystyle{plainnat}

\newpage

\section*{NeurIPS Paper Checklist}

\begin{enumerate}

\item {\bf Claims}
    \item[] Answer: \answerYes{}
    \item[] Justification: The abstract and introduction state the architectural contributions and experimental results with specific numbers. Limitations (datasets, scale, training cost) are discussed explicitly in Section~\ref{sec:conclusion}.

\item {\bf Limitations}
    \item[] Answer: \answerYes{}
    \item[] Justification: Section~\ref{sec:conclusion} includes a dedicated Limitations paragraph covering dataset scope, scale, training cost, and single-run reporting.

\item {\bf Theory assumptions and proofs}
    \item[] Answer: \answerYes{}
    \item[] Justification: All theorems and propositions are numbered and cross-referenced. Assumptions are stated explicitly in the appendix proofs (Appendix~\ref{app:proof-necessity}, \ref{app:proof-streaming}, \ref{app:proof-convergence}, \ref{app:proof-depth}). Main-body statements reference the appendix for full proofs.

\item {\bf Experimental result reproducibility}
    \item[] Answer: \answerYes{}
    \item[] Justification: The architecture is fully described in Section~\ref{sec:method}. Hyperparameters, optimizer settings, learning rate schedules, and training details are provided in Appendix~\ref{app:exp-details}. All experiments use publicly available data (OpenWebText, Wikipedia).

\item {\bf Open access to data and code}
    \item[] Answer: \answerNo{}
    \item[] Justification: Code is not released at submission time due to patent considerations. The architecture and training procedure are described in sufficient detail to reproduce.

\item {\bf Experimental setting/details}
    \item[] Answer: \answerYes{}
    \item[] Justification: Section~\ref{sec:exp-setup} describes data, tokenization, context lengths, optimizer, and baselines. Appendix~\ref{app:exp-details} provides full hyperparameter tables.

\item {\bf Experiment statistical significance}
    \item[] Answer: \answerNo{}
    \item[] Justification: All results are single runs. We acknowledge this explicitly and report trends across multiple configurations (widths, depths, block sizes) rather than relying on individual comparisons.

\item {\bf Experiments compute resources}
    \item[] Answer: \answerYes{}
    \item[] Justification: Section~\ref{sec:exp-setup} reports FLOPs per token and training iterations. Inference timing is measured on a single A100 GPU (Appendix~\ref{app:inference-timing}).

\item {\bf Code of ethics}
    \item[] Answer: \answerYes{}
    \item[] Justification: The research conforms with the NeurIPS Code of Ethics. No human subjects, private data, or dual-use concerns.

\item {\bf Broader impacts}
    \item[] Answer: \answerNA{}
    \item[] Justification: This is foundational architecture research. The work reduces inference cost for language models, which has broadly positive efficiency implications but no direct path to specific negative applications beyond those inherent to language modeling in general.

\item {\bf Safeguards}
    \item[] Answer: \answerNA{}
    \item[] Justification: No pretrained language models or scraped datasets are released.

\item {\bf Licenses for existing assets}
    \item[] Answer: \answerYes{}
    \item[] Justification: OpenWebText is cited~\citep{gokaslan2019openwebtext}. Wikipedia is publicly available. All referenced works are cited.

\item {\bf New assets}
    \item[] Answer: \answerNA{}
    \item[] Justification: No new datasets, models, or code are released with this submission.

\item {\bf Crowdsourcing and research with human subjects}
    \item[] Answer: \answerNA{}
    \item[] Justification: No crowdsourcing or human subjects research.

\item {\bf Institutional review board (IRB) approvals or equivalent for research with human subjects}
    \item[] Answer: \answerNA{}
    \item[] Justification: No human subjects research.

\item {\bf Declaration of LLM usage}
    \item[] Answer: \answerNA{}
    \item[] Justification: LLMs were not used as a component of the core methodology.

\end{enumerate}

\newpage
\appendix
\renewcommand{\thetable}{A\arabic{table}}
\renewcommand{\theHtable}{A\arabic{table}}
\setcounter{table}{0}
\renewcommand{\thefigure}{A\arabic{figure}}
\renewcommand{\theHfigure}{A\arabic{figure}}
\setcounter{figure}{0}

\section{Full Proofs}
\label{app:proofs}

\textbf{Notation.} Throughout the appendix, $e$ denotes the base token embeddings, $G$ denotes the full correction operator, and $c^{(k)}$ denotes the correction vector at iteration $k$. The fixed point is $c^* = G(c^*)$.

\subsection{Proof of Theorem~\ref{thm:necessity} (Structural Characterization)}
\label{app:proof-necessity}

\textbf{Formal statement.}

\emph{Setup.} Let $e_t \in \mathbb{R}^d$ denote the base embedding at position $t$, let $\texttt{Block}$ denote a $D$-layer transformer block with causal attention, and let $F: \mathbb{R}^d \times \mathbb{R}^d \to \mathbb{R}^d$ be a correction function. The architecture processes position $t$ as follows: form the corrected input $\tilde{x}_t = e_t + F(e_t, z_{t-1})$, compute $z_t = \texttt{Block}(\tilde{x}_t;\, \tilde{x}_{<t})$, and cache $z_t$ for future positions.

\emph{Assumptions.}
\begin{enumerate}[label=({\Roman*}),leftmargin=*,itemsep=4pt]
\item \emph{Additive correction.} The corrected input has the form $\tilde{x}_t = e_t + F(e_t, z_{t-1})$, where $F$ is a continuous correction function that maps into a bounded ball $\overline{B}(0, B) \subset \mathbb{R}^d$.
\item \emph{Single-pass streaming.} During inference, tokens are processed left-to-right, one at a time. The correction at position $t$ uses the cached block output $z_{t-1}$ from the previous position and the current embedding $e_t$. By causality, $z_{t-1}$ depends only on positions $\leq t{-}1$.
\end{enumerate}

\emph{Conclusions.}
\begin{enumerate}[label=(\alph*),leftmargin=*,itemsep=4pt]
\item \emph{Past-only.} By streaming (Property~II), the correction at position $t$ depends only on previously computed outputs: $F(e_t, z_{t-1})$ where $z_{t-1}$ encodes positions $0, \ldots, t{-}1$.
\item \emph{Non-cumulative.} Among additive unrolling strategies, only the non-cumulative form $\tilde{x}_t^{(k)} = e_t + F(e_t, z_{t-1}^{(k)})$ is compatible with streaming. The cumulative (resnet) alternative $\tilde{x}_t^{(k)} = \tilde{x}_t^{(k-1)} + F(\tilde{x}_t^{(k-1)}, z_{t-1}^{(k)})$ fails because $F(\tilde{x}_t^*, z_{t-1}^*) = 0$ at the fixed point.
\item \emph{Existence and uniqueness.} The non-cumulative past-only system has a unique fixed point, and streaming computes it exactly.
\end{enumerate}

\begin{proof}
\textbf{Part (a): Past-only.} Streaming (Property~II) processes tokens left-to-right. When token $t$ arrives, the correction uses $z_{t-1}$ (cached from the previous token) and $e_t$. Since causal attention ensures $z_{t-1}$ depends only on positions $\leq t{-}1$, the correction at position $t$ is a function of past outputs only.

\textbf{Part (b): Non-cumulative.} Given the additive correction form (Property~I), there are two ways to unroll $K$ training steps:

\emph{Non-cumulative:} $\tilde{x}_t^{(k)} = e_t + F(e_t, z_{t-1}^{(k)})$. Each step recomputes the correction from the base embedding $e_t$. At the fixed point, $F(e_t, z_{t-1}^*) = c_t^*$, a nonzero correction. In streaming, past outputs are exact (from cache), so a single evaluation gives $\tilde{x}_t = e_t + F(e_t, z_{t-1}^*) = e_t + c_t^* = \tilde{x}_t^*$. Exact.

\emph{Cumulative (resnet):} $\tilde{x}_t^{(k)} = \tilde{x}_t^{(k-1)} + F(\tilde{x}_t^{(k-1)}, z_{t-1}^{(k)})$. Each step adds an increment to the previous output. At the fixed point, $\tilde{x}_t^* = \tilde{x}_t^* + F(\tilde{x}_t^*, z_{t-1}^*)$, which forces $F(\tilde{x}_t^*, z_{t-1}^*) = 0$: the correction function learns to output zero at convergence. In streaming, the single step starts from $\tilde{x}_t^{(0)} = e_t$ and gives $\tilde{x}_t = e_t + F(e_t, z_{t-1}^*)$. But $F$ was trained to vanish at $\tilde{x}_t^*$, not at $e_t$, so the result is neither zero nor the correct fixed point $\tilde{x}_t^*$.

Therefore, among additive unrolling strategies, the non-cumulative form is the one that reproduces the correct fixed-point output during streaming.

\textbf{Part (c): Existence and uniqueness.} By part~(a), the system is triangular: $\tilde{x}_1^*$ is determined first (from $z_0 = \mathbf{0}$ and $e_1$), then $z_1^*$, then $\tilde{x}_2^*$, and so on. Each step is a deterministic evaluation, so the fixed point exists, is unique, and is constructively computable by left-to-right evaluation---exactly the streaming computation.
\end{proof}

\subsection{Proof of Theorem~\ref{thm:exact-streaming} (Exact Streaming)}
\label{app:proof-streaming}

\textbf{Formal statement.} Let $c^{*(T)}$ be the unique fixed point for a sequence of length $T$. Assume \emph{prefix consistency}: for any $T' > T$ and $t \leq T$, the correction operator satisfies $G_t^{(T')} = G_t^{(T)}$, i.e., appending tokens beyond position $T$ does not change the operator at earlier positions. (This holds by construction for any causal architecture where the operator at position $t$ depends only on positions $\leq t$.) Then:
\begin{enumerate}[label=(\alph*),leftmargin=*,itemsep=4pt]
\item \emph{Prefix invariance.} $c^{*(T')}_t = c^{*(T)}_t$ for $t \leq T$ and any $T' > T$.
\item \emph{Exactness.} If past corrections are at the fixed point, the streaming operator produces the exact fixed-point correction for the new token.
\item \emph{No contraction needed.} Exactness holds without any contraction assumption.
\end{enumerate}

The following lemma establishes that the Jacobi iteration converges in finitely many steps, which is the foundation for the streaming exactness proof.

\begin{lemma}[Finite-step exactness]
\label{thm:finite-step}
Let $G$ be a past-only correction operator. Then:
\textup{(a)}~The fixed point $c^*$ exists and is unique, without contraction.
\textup{(b)}~After $k$ Jacobi iterations from any $c^{(0)}$: $c_t^{(k)} = c_t^*$ for all $t \leq k$.
\textup{(c)}~The iteration reaches the exact fixed point after at most $T$ steps: $c^{(T)} = c^*$.
\end{lemma}

\begin{proof}
(a) By construction: $G_1$ is a constant, so $c_1^*$ is determined. Given $c_1^*, \ldots, c_{t-1}^*$, we have $c_t^* = G_t(c_1^*, \ldots, c_{t-1}^*)$ uniquely.

(b) By induction. Base: $c_1^{(1)} = G_1 = c_1^*$. Step: assume $c_s^{(k)} = c_s^*$ for $s \leq k$. Then $c_{k+1}^{(k+1)} = G_{k+1}(c_1^{(k)}, \ldots, c_k^{(k)}) = G_{k+1}(c_1^*, \ldots, c_k^*) = c_{k+1}^*$.

(c) Set $k = T$ in (b).
\end{proof}

\begin{proof}[Proof of Theorem~\ref{thm:exact-streaming}]
(a) By prefix consistency, $G_t^{(T')} = G_t^{(T)}$ for $t \leq T$. Hence the fixed-point equations for positions $1, \ldots, T$ are identical in the length-$T$ and length-$T'$ systems, so $c^{*(T')}_t = c^{*(T)}_t$.

(b) $c^{*(T+1)}_{T+1} = G^{(T+1)}_{T+1}(c^{*(T+1)}_1, \ldots, c^{*(T+1)}_T)$. By (a), $c^{*(T+1)}_t = c^{*(T)}_t$ for $t \leq T$. The streaming operator computes exactly $G^{(T+1)}_{T+1}$ using cached corrections $c^{*(T)}$. Concretely, the abstract operator $G_t$ is realized by the correction FFN applied to the cached block output $z_{t-1}$ and the current token embedding $e_t$: once earlier positions are at their fixed-point values, $z_{t-1}$ is fully determined, and the streaming step computes $c_t = G_t(c_1^*, \ldots, c_{t-1}^*; e_t)$ exactly.

(c) By induction from the base case and part (b).
\end{proof}

\subsection{Proof of Theorem~\ref{thm:convergence} (Convergence)}
\label{app:proof-convergence}

\textbf{Formal statement.} Let $G_t(c_{<t}; e)$ denote the full correction operator at position $t$. Assume $\|G(c; e) - G(c'; e')\| \leq L\|c - c'\| + M\|e - e'\|$ for all $c, c', e, e'$. If $L < 1$, then:
\begin{enumerate}[label=(\alph*),leftmargin=*,itemsep=4pt]
\item \emph{Contraction.} $\|c^{(k)} - c^*\| \leq L^k \|c^{(0)} - c^*\|$.
\item \emph{Warm-start bound.} $\|G(c^*; e') - c^{*\prime}\| \leq \frac{LM}{1-L}\|e' - e\|$.
\end{enumerate}

\begin{proof}
\textbf{Part (a).} By the Banach fixed-point theorem with contraction constant $L < 1$.

\textbf{Part (b).} $\|c^* - c^{*\prime}\| \leq M\|e - e'\| + L\|c^* - c^{*\prime}\|$, so $\|c^* - c^{*\prime}\| \leq \frac{M}{1-L}\|e' - e\|$. Then $\|G(c^*; e') - c^{*\prime}\| \leq L\|c^* - c^{*\prime}\| \leq \frac{LM}{1-L}\|e' - e\|$.
\end{proof}

Theorem~\ref{thm:convergence} assumes $L < 1$ but does not say how to verify this from the per-position Jacobian structure. The following lemma provides a practical bound: given bounds on the partial derivatives $\|\partial G_t / \partial c_s\|$, one can choose a weighted norm that makes the global Lipschitz constant explicit.

\begin{lemma}[Causal contraction bound]
\label{thm:causal-contraction}
Let $G(c; e)$ be the full-sequence correction operator with past-only dependencies, and suppose $\|\partial G_t / \partial c_s\|_\textup{op} \leq a_{t,s}$ for $s < t$, where $\|A\|_\textup{op} = \sup_{\|x\|=1} \|Ax\|$ is the operator norm. For positive weights $w = (w_1, \ldots, w_T)$, define $\|c\|_w = \max_t w_t\|c_t\|$. Then $G$ is $L_w$-Lipschitz in $\|\cdot\|_w$ with constant:
\[
L_w = \max_{1 \leq t \leq T} w_t \sum_{s=1}^{t-1} \frac{a_{t,s}}{w_s}.
\]
In particular, if $L_w < 1$, then $G$ is a contraction.
\end{lemma}

\begin{proof}
For each $t$: $\|G_t(c) - G_t(c')\| \leq \sum_{s < t} a_{t,s}\|c_s - c_s'\| \leq \sum_{s < t} \frac{a_{t,s}}{w_s}\|c - c'\|_w$. Multiplying by $w_t$ and taking the maximum over $t$ gives $\|G(c) - G(c')\|_w \leq L_w\|c - c'\|_w$.
\end{proof}

When $K$ is chosen at training time, one may want to know how close $c^{(K)}$ is to $c^*$ without computing $c^*$. The next lemma gives a computable bound using only consecutive iterates.

\begin{lemma}[A posteriori error bound]
If $L < 1$, then after $k$ iterations:
$\|c^{(k)} - c^*\| \leq \frac{L}{1-L}\|c^{(k)} - c^{(k-1)}\|$.
\end{lemma}

\begin{proof}
By the triangle inequality, $\|c^{(k)} - c^*\| \leq \sum_{j=1}^{\infty} \|c^{(k+j)} - c^{(k+j-1)}\| \leq \sum_{j=1}^{\infty} L^j \|c^{(k)} - c^{(k-1)}\| = \frac{L}{1-L}\|c^{(k)} - c^{(k-1)}\|$.
\end{proof}

Theorem~\ref{thm:convergence}(a) gives a global contraction rate $L^k$ that treats all positions uniformly, but this is overly pessimistic. Because $G$ is past-only, its Jacobian is strictly lower-triangular and therefore nilpotent, so position $t$ reaches its exact fixed point after at most $t$ iterations rather than $T$. The following lemma exploits this structure to give a tighter, position-dependent error bound via causal path sums.

\begin{lemma}[Finite-depth error bound]
\label{thm:path-sum-error}
Let $G$ be a past-only correction operator with Jacobian bounds $a_{t,s}$. Let $A$ be the strictly lower-triangular matrix with entries $a_{t,s}$, and $B = \max_t \|G_t(0)\|$. After $N$ Jacobi iterations from $c^{(0)} = 0$:
\[
\|c_t^{(N)} - c_t^*\| \leq \sum_{k=N}^{T-1} [A^k]_{t,\cdot} \mathbf{1} \cdot B = [(A^N(I-A)^{-1})_{t,\cdot}] \mathbf{1} \cdot B
\]
\end{lemma}

\begin{proof}
By the integral form of the mean value theorem, the error $e^{(k)} = c^{(k)} - c^*$ satisfies $e^{(k+1)}_t = \sum_{s < t} J_{t,s}^{(k)} e^{(k)}_s$ where $\|J_{t,s}^{(k)}\| \leq a_{t,s}$ by the Jacobian bound assumption. Bounding by the entrywise matrix $A$ and iterating from $e^{(0)} = -c^*$ gives $\|e^{(N)}_t\| \leq [A^N |c^*|]_t$. To bound $|c^*|$: the fixed point satisfies $c_t^* = G_t(c_{<t}^*)$, so $\|c_t^*\| \leq \|G_t(0)\| + \sum_{s < t} a_{t,s}\|c_s^*\|$, i.e., $|c^*| \leq B \cdot \mathbf{1} + A|c^*|$, hence $|c^*| \leq (I - A)^{-1} B \cdot \mathbf{1}$ (well-defined since $A$ is nilpotent). Substituting: $\|e^{(N)}_t\| \leq [A^N (I-A)^{-1}]_{t,\cdot} \mathbf{1} \cdot B$.
\end{proof}

\subsection{Proof of Proposition~\ref{prop:depth} (Depth Separation)}
\label{app:proof-depth}

\textbf{Formal statement.}

\emph{Setup.} Suppose the data is generated by a process with state update $\mathbf{s}_t = h^*(\mathbf{s}_{t-1}, \mathbf{e}_t)$, where $\mathbf{s}_t \in \mathbb{R}^m$ is the process state and $\mathbf{e}_t = (\mathbf{e}_{t-W+1}, \ldots, \mathbf{e}_t) \in \mathbb{R}^{CW}$ collects the token embeddings in the current window. The context-ready architecture (Equations~\ref{eq:corr-add}--\ref{eq:block}) with attention window $W$ maintains a state $\hat{\mathbf{b}}_t \in \mathbb{R}^{nW}$ over the full window of $W$ positions, where $n$ is the per-position state dimension. The state evolves via
$\hat{\mathbf{b}}_t = h(\hat{\mathbf{b}}_{t-1}, \mathbf{e}_t)$,
where $h: \mathbb{R}^{nW} \times \mathbb{R}^{CW} \to \mathbb{R}^{nW}$ composes the correction FFN and block (Equations~\ref{eq:corr-add}--\ref{eq:block}). Let $L_h$ be the Lipschitz constant of $h$ in its first argument.

To compare these two systems, let $\pi: \mathbb{R}^m \to \mathbb{R}^{nW}$ be a map from the process state to the architecture's state space. The architecture faithfully tracks the process when the following diagram commutes: advancing the process state by $h^*$ and then projecting gives the same result as projecting and then advancing by $h$. The commutation error
\[
\varepsilon_D \;=\; \sup_{\mathbf{s},\, \mathbf{e}}\; \bigl\|h\bigl(\pi(\mathbf{s}),\, \mathbf{e}\bigr) - \pi\bigl(h^*(\mathbf{s},\, \mathbf{e})\bigr)\bigr\|
\]
measures how far the diagram is from commuting: it captures both the block's finite-depth approximation error and any information lost when the two state spaces differ.

\emph{Assumptions.}
\begin{enumerate}[label=(\roman*),leftmargin=*,itemsep=2pt]
\item $L_h < 1$ and $\varepsilon_D < \infty$.
\item \emph{Prediction sufficiency.} $\pi$ preserves prediction-relevant information: $p(x_{t+1} \mid \mathbf{s}_t) = p(x_{t+1} \mid \pi(\mathbf{s}_t))$.
\item \emph{Lipschitz readout.} The prediction function $\phi: \mathbb{R}^{nW} \to \Delta$ is $L_\phi$-Lipschitz.
\end{enumerate}

\emph{Conclusions.}
\begin{enumerate}[label=(\alph*),leftmargin=*,itemsep=4pt]
\item \emph{Context-ready error bound.} The accumulated state error satisfies $\|\hat{\mathbf{b}}_t - \pi(\mathbf{s}_t)\| \leq \varepsilon_D / (1 - L_h)$ uniformly in $t$. By prediction sufficiency and Lipschitz readout, the prediction error is bounded by $L_\phi \varepsilon_D / (1 - L_h)$. If $\varepsilon_D = 0$, then streaming is exact for all $t$.

\item \emph{Standard transformer receptive-field bound.} Let $N$ be the number of layers in a standard transformer with attention window $W$. Then position $t$ has no computational path to any position before $t - NW$, so the transformer cannot represent any function whose output at position $t$ depends on inputs before $t - NW$.

\end{enumerate}

\begin{proof}
\textbf{Part (a).} At step $t$, the architecture computes $\hat{\mathbf{b}}_t = h(\hat{\mathbf{b}}_{t-1}, \mathbf{e}_t)$ while the projected true state satisfies $\pi(\mathbf{s}_t) = \pi(h^*(\mathbf{s}_{t-1}, \mathbf{e}_t))$. By the triangle inequality:
\[
\|\hat{\mathbf{b}}_t - \pi(\mathbf{s}_t)\| \;\leq\; \underbrace{\|h(\hat{\mathbf{b}}_{t-1}, \mathbf{e}_t) - h(\pi(\mathbf{s}_{t-1}), \mathbf{e}_t)\|}_{\leq\; L_h\|\hat{\mathbf{b}}_{t-1} - \pi(\mathbf{s}_{t-1})\|} \;+\; \underbrace{\|h(\pi(\mathbf{s}_{t-1}), \mathbf{e}_t) - \pi(h^*(\mathbf{s}_{t-1}, \mathbf{e}_t))\|}_{\leq\; \varepsilon_D}.
\]
Unrolling with $\hat{\mathbf{b}}_0 = \pi(\mathbf{s}_0)$ gives $\|\hat{\mathbf{b}}_t - \pi(\mathbf{s}_t)\| \leq \varepsilon_D \sum_{j=0}^{t-1} L_h^j \leq \varepsilon_D/(1 - L_h)$.

\textbf{Part (b).} With attention window $W$, the output of layer $\ell$ at position $t$ depends only on positions in $[t - \ell W,\, t]$. After $N$ layers, position $t$ has no computational path to any position before $t - NW$.
\end{proof}

\begin{remark}[Depth separation]
Parts (a) and (b) together suggest a depth separation: the context-ready architecture propagates context through the correction chain at no additional depth cost, while a standard windowed transformer must allocate layers for propagation. When propagation and local computation cannot be interleaved---as in the pointer-chasing task where each hop requires a separate lookup---the standard transformer needs at least $\lceil T/W \rceil$ additional layers. In general, some layers may serve both roles, so $N_{\mathrm{map}} + \lceil T/W \rceil$ is an upper bound on the required depth.
\end{remark}

\section{Pointer Chasing Details}
\label{app:pointer-chasing}

\textbf{Motivation.} Fixed-depth transformers are confined to $\mathrm{TC}^0$~\citep{merrill2023parallelism}: an $N$-layer transformer can compose at most $N$ sequential reasoning steps in a single forward pass. We design a synthetic task that directly tests this depth limit. Answering a query requires chaining a variable number of sequential lookups, so a model that can only perform a fixed number of parallel steps will fail once the required chain length exceeds its depth. The context-ready architecture sidesteps this barrier because its recurrent correction chain provides sequential computation at inference, even with a single block ($D{=}1$).

\textbf{Task definition.} The pointer-chasing task has $H$ hops and $M$ keys per level. The input contains a \emph{base table} (level 0) mapping $M$ keys to values, followed by $H$ \emph{index tables} (levels $1, \ldots, H$), each mapping $M$ keys to keys of the previous level via random permutations (bijections).  After each table, a \emph{query section} provides dense targets: a triplet \texttt{Q key answer} for every key at every level defined so far.  Resolving a query at level $\ell$ requires $\ell$ sequential lookups.

\textbf{Worked example} ($H{=}2$, $M{=}3$, 10 values).  The base table maps A$\to$v3, B$\to$v0, C$\to$v8.  Index table~1 maps D$\to$A, E$\to$B, F$\to$C.  Index table~2 maps G$\to$D, H$\to$E, I$\to$F.  The encoding uses reversed triplets (\texttt{value=key}) so that causal attention can see the value to the left of the key:

\begin{center}
\small
\textbf{Level 0 (base table + queries):}\\[2pt]
\begin{tabular}{@{}r@{\;}cccccccccccc@{}}
Input:  & v3 & = & A & v0 & = & B & v8 & = & C & $|$ & & \\
Target: & \_ & \_ & \_ & \_ & \_ & \_ & \_ & \_ & \_ & \_ & & \\[2pt]
Input:  & Q & A & v3 & Q & B & v0 & Q & C & v8 & $|$ & & \\
Target: & \_ & \textbf{v3} & \_ & \_ & \textbf{v0} & \_ & \_ & \textbf{v8} & \_ & \_ & & \\
\end{tabular}

\smallskip
\textbf{Level 1 (index table + queries):}\\[2pt]
\begin{tabular}{@{}r@{\;}cccccccccccc@{}}
Input:  & A & = & D & B & = & E & C & = & F & $|$ & & \\
Target: & \_ & \_ & \_ & \_ & \_ & \_ & \_ & \_ & \_ & \_ & & \\[2pt]
Input:  & Q & D & v3 & Q & E & v0 & Q & F & v8 & $|$ & & \\
Target: & \_ & \textbf{v3} & \_ & \_ & \textbf{v0} & \_ & \_ & \textbf{v8} & \_ & \_ & & \\
\end{tabular}

\smallskip
\textbf{Level 2 (index table + final query):}\\[2pt]
\begin{tabular}{@{}r@{\;}ccccccccccccc@{}}
Input:  & D & = & G & E & = & H & F & = & I & $|$ & Q & G \\
Target: & \_ & \_ & \_ & \_ & \_ & \_ & \_ & \_ & \_ & \_ & \_ & \textbf{v3} \\
\end{tabular}
\end{center}

\noindent
Targets (bold) appear only at key positions in query sections.  Level-0 queries are trivial lookups (Q\;A $\to$ v3).  Level-1 queries require one composition (Q\;D $\to$ A $\to$ v3).  Level-2 queries require two compositions (Q\;G $\to$ D $\to$ A $\to$ v3).  The final token is the actual test query with no answer provided in the input.  Dense targets at every level are essential: without them, the model cannot learn multi-hop composition even with BPTT.

Each level uses its own key namespace (A, B, C at level 0; D, E, F at level 1; G, H, I at level 2; etc.) to prevent ambiguity.  Key ordering within each table is fixed (not shuffled), so the model can exploit positional patterns via RoPE.

\textbf{Settings.} $H = 10$ hops, $M = 5$ keys, 10 values, embedding dimension $C = 256$, 4 attention heads, batch size 64, window size 38, fixed key ordering, per-level key tokens, RoPE attention. Learning rate $1 \times 10^{-4}$.

\textbf{Why windowed attention.} Without windowed attention, all models---including deep transformers---can directly attend from any query position to the base table, achieving ${\sim}1/M$ accuracy without genuine composition.  Windowed attention ($w = 38$) ensures that higher-level query sections cannot see the base table, forcing the model to chain through intermediate levels.  This reveals the true depth-limited structure of fixed-depth transformers.

\textbf{Wave propagation in BPTT.} The $D{=}1$ model solves levels sequentially: level 0 converges first, then level 1, then 2, and so on. This is visible in the training dynamics (see progression below). The wave pattern is consistent with corrections propagating through the recurrent chain.

\textbf{BPTT progression.} The progression below uses a smaller configuration ($C{=}128$, lr $= 1\text{e-3}$, 20 values) to demonstrate wave propagation at reduced compute:

\begin{table}[H]
\centering
\small
\begin{tabular}{lccccccccccc}
\toprule
\textbf{Iter} & L0 & L1 & L2 & L3 & L4 & L5 & L6 & L7 & L8 & L9 & L10 \\
\midrule
3K & 1.00 & 0.74 & 0.27 & 0.20 & 0.19 & 0.18 & 0.18 & 0.18 & 0.17 & 0.16 & 0.16 \\
8K & 1.00 & 1.00 & 0.80 & 0.59 & 0.36 & 0.21 & 0.21 & 0.19 & 0.20 & 0.18 & 0.15 \\
13.5K & 1.00 & 1.00 & 0.99 & 0.99 & 0.94 & 0.83 & 0.51 & 0.24 & 0.20 & 0.20 & 0.18 \\
23K & 1.00 & 1.00 & 1.00 & 1.00 & 1.00 & 0.99 & 0.99 & 0.99 & 0.99 & 0.98 & 0.98 \\
\bottomrule
\end{tabular}
\end{table}

\textbf{20-hop scaling.} At $C = 512$ with $\text{lr} = 1\text{e-4}$, the same $D{=}1$ architecture solves all 21 levels (20 hops) in ${\sim}4$K iterations, confirming that the recurrent mechanism scales to deeper composition chains.

\section{Extended Experimental Details}
\label{app:exp-details}

\subsection{Hyperparameters}

\begin{table}[H]
\centering
\caption{Hyperparameters by experiment.}
\small
\begin{tabular}{lcccc}
\toprule
\textbf{Hyperparameter} & ${\sim}85$M & ${\sim}340$M & Token-matched & Width scaling \\
\midrule
Block size & 256 & 256 & 64 & 64 \\
Batch size & 64 & 32 & 1024 & 1024 \\
Learning rate & $2 \times 10^{-4}$ & $2 \times 10^{-4}$ & $2 \times 10^{-4}$ & $2 \times 10^{-4}$ \\
Training iters & 100K & 200K & varies & varies \\
$K$ / $k_{\min}$ & 5 / 2 & 5 / 2 & 5 / 2 & 5 / 2 \\
Dropout & 0.2 & 0.2 & 0.2 & 0.2 \\
Vocab size & 32,000 & 32,000 & 32,000 & 32,000 \\
\bottomrule
\end{tabular}
\end{table}

\subsection{Ablations}
\label{app:ablations}

\begin{table}[H]
\centering
\caption{Correction FFN ablation ($C = 446$, Wikipedia, BPE 16k). ``Roformer-hFFN'' denotes a standard roformer with an extra FFN layer to match the FLOP cost of the correction FFN.}
\small
\begin{tabular}{lccc}
\toprule
\textbf{Model} & \textbf{FLOPs/tok} & \textbf{Val PPL} & \textbf{Seq.\ $K{=}1$} \\
\midrule
D=3 corr\_ffn $K{=}5$ & $44C^2$ & \textbf{23.98} & 23.96 \\
Roformer-hFFN $N{=}3$ & $44C^2$ & 25.78 & --- \\
D=3 block\_head (no corr\_ffn) & $36C^2$ & 27.32 & 28.46 \\
Roformer $N{=}3$ & $36C^2$ & 27.19 & --- \\
\bottomrule
\end{tabular}
\end{table}

\begin{table}[H]
\centering
\caption{Token-aware correction variants ($C = 446$, Wikipedia).}
\small
\begin{tabular}{llcccc}
\toprule
$D$ & \textbf{Variant} & \textbf{FLOPs} & \textbf{PPL} & \textbf{Seq $K{=}1$} & $L$ \\
\midrule
\multirow{3}{*}{2}
& corr\_ffn (token-blind) & $32C^2$ & 26.68 & 26.72 & 0.74 \\
& corr\_ffn\_add & $32C^2$ & \textbf{26.09} & 26.48 & 0.54 \\
& corr\_ffn\_concat & $36C^2$ & \textbf{25.48} & 25.82 & 0.54 \\
\midrule
\multirow{3}{*}{3}
& corr\_ffn (token-blind) & $44C^2$ & 23.98 & 23.96 & --- \\
& corr\_ffn\_add & $44C^2$ & \textbf{23.79} & 24.12 & 0.55 \\
& corr\_ffn\_concat & $48C^2$ & \textbf{23.41} & 23.73 & 0.74 \\
\bottomrule
\end{tabular}
\end{table}

\begin{table}[H]
\centering
\caption{Scale dependence ($C = 50$--$768$, Wikipedia).}
\small
\begin{tabular}{clccc}
\toprule
$C$ & \textbf{Comparison} & \textbf{Context-Ready} & \textbf{Baseline} & \textbf{$\Delta$} \\
\midrule
50  & D=3 vs.\ Roformer-hFFN $N{=}3$ & 84.3 & 83.0 & +1.3 \\
74  & D=3 vs.\ Roformer-hFFN $N{=}3$ & 62.1 & 61.4 & +0.7 \\
446 & D=3 vs.\ Roformer $N{=}4$ & 23.79 & 24.85 & $-$1.06 \\
768 & D=3 vs.\ Roformer $N{=}4$ & 18.66 & 20.05 & $-$1.39 \\
\bottomrule
\end{tabular}
\end{table}

\noindent At very small widths ($C \leq 74$), the correction FFN's overhead outweighs its benefit; the correction advantage emerges at moderate widths ($C \geq 446$) and grows with scale. All main-body claims are based on results at $C \geq 256$.

\begin{table}[H]
\centering
\caption{$k_{\min}$ ablation ($C = 50$, $D{=}1$, $K = 10$).}
\small
\begin{tabular}{lcc}
\toprule
\textbf{Metric} & \textbf{$k_{\min} = 2$} & \textbf{No $k_{\min}$} \\
\midrule
Val PPL ($K = 10$) & 84.32 & \textbf{84.16} \\
Seq $K = 1$ & 84.61 & 84.19 \\
Parallel $K = 1$ & \textbf{118.35} & 130.95 \\
Empirical $L$ & \textbf{0.72} & 0.94 \\
\bottomrule
\end{tabular}
\end{table}

\subsection{$K = 10$ Training Details}
\label{app:k10-details}

All $K{=}10$ experiments: block size 1024, $\text{lr} = 2\text{e-4}$, softmax attention, $n_\text{head} = 16$, OWT.

\begin{table}[H]
\centering
\caption{$D{=}1$ $K{=}10$ (fixed, no $k_{\min}$, batch 16) vs.\ roformer $N{=}6$ (batch 16), block size 1024.}
\small
\begin{tabular}{lrrl}
\toprule
\textbf{Iter} & \textbf{$N{=}6$} & \textbf{$D{=}1$ $K{=}10$} & \textbf{Gap} \\
\midrule
40K & 43.24 & 43.83 & +0.59 \\
60K & 39.21 & 39.27 & +0.06 \\
65K & 38.55 & 38.49 & $-$0.06 \\
80K & 36.99 & 36.38 & $-$0.61 \\
100K & 35.37 & 34.40 & $-$0.97 \\
\bottomrule
\end{tabular}
\end{table}

\subsection{Sequential $K{=}1$ Validation}
\label{app:seq-k1}

Full depth progression for all configurations, confirming Theorem~\ref{thm:exact-streaming}.

\begin{table}[H]
\centering
\caption{$D{=}1$ $C{=}2048$ depth progression across block sizes (OWT, 200K--400K iters).}
\small
\begin{tabular}{llcccccc}
\toprule
\textbf{Block size} & \textbf{Iters} & \textbf{Par.\ $K{=}1$} & \textbf{Par.\ $K{=}2$} & \textbf{Par.\ $K{=}3$} & \textbf{Par.\ $K{=}5$} & \textbf{Par.\ $K{=}10$} & \textbf{Seq.\ $K{=}1$} \\
\midrule
256  & 100K & 84.44 & 43.77 & 40.52 & 39.95 & 40.02 & 40.03 \\
256  & 400K & 70.80 & 36.21 & 33.23 & 32.70 & 32.79 & 32.80 \\
512  & 100K & 81.02 & 38.25 & 35.01 & 34.35 & 34.41 & 34.43 \\
512  & 200K & 73.22 & 34.33 & 31.24 & 30.61 & 30.68 & 30.69 \\
1024 & 200K & 84.13 & 34.42 & 30.39 & 29.51 & 29.43 & \textbf{29.43} \\
\bottomrule
\end{tabular}
\end{table}

\begin{table}[H]
\centering
\caption{Higher-$D$ depth progression ($C{=}1024$, block size 256, OWT).}
\small
\begin{tabular}{llccccc}
\toprule
$D$ & \textbf{Par.\ $K{=}1$} & \textbf{Par.\ $K{=}2$} & \textbf{Par.\ $K{=}3$} & \textbf{Par.\ $K{=}5$} & \textbf{Par.\ $K{=}10$} & \textbf{Seq.\ $K{=}1$} \\
\midrule
5  & 58.17 & 40.18 & 38.80 & 38.60 & 38.61 & 38.62 \\
8  & 51.60 & 40.12 & 39.22 & 39.10 & 39.10 & 39.10 \\
12 & 38.33 & 32.58 & 32.30 & 32.29 & 32.29 & 32.29 \\
23 & 32.80 & 29.03 & 28.89 & 28.88 & 28.88 & 28.88 \\
\bottomrule
\end{tabular}
\end{table}

At higher $D$, convergence is faster: $D{=}12$ and $D{=}23$ match $K{=}5$ within 0.01~PPL at $K{=}3$; $D{=}8$ is within 0.12~PPL. Parallel $K{=}1$ ratio to actual quality shrinks with $D$ (from $2.85\times$ at $D{=}1$ to $1.14\times$ at $D{=}23$), confirming that the correction mechanism accounts for a larger fraction of quality at low $D$.

\subsection{Block Size Scaling}
\label{app:block-size}

Longer context lengths give the $D{=}1$ correction chain more sequential steps to accumulate depth, so the gap to $N{=}6$ should shrink with block size. Table~\ref{tab:block-size} confirms this: at $K{=}5$, the gap narrows from $+1.19$ at block size 256 to $+0.31$ at 1024. With $K{=}10$ at block size 1024, $D{=}1$ overtakes $N{=}6$ entirely (Section~\ref{sec:exp-d1}).

\begin{table}[H]
\centering
\caption{$D{=}1$ $C{=}2048$ vs.\ $N{=}6$ $C{=}1088$ gap at 200K iterations across block sizes (OWT).}
\label{tab:block-size}
\small
\begin{tabular}{lrrr}
\toprule
\textbf{Block size} & \textbf{$N{=}6$ PPL} & \textbf{$D{=}1$ PPL} & \textbf{Gap} \\
\midrule
256  & 34.15 & 35.34 & +1.19 \\
512  & 30.11 & 30.57 & +0.46 \\
1024 & 29.22 & 29.53 & +0.31 \\
\bottomrule
\end{tabular}
\end{table}

\subsection{Token-Matched Training Curves}
\label{app:token-matched}

To isolate the effect of the correction mechanism from FLOP differences, we compare $D{=}x$ context-ready against $N{=}x$ standard transformers at the same embedding dimension ($C = 1024$, block size 64), so both see the same number of tokens per training iteration. At every depth tested, the context-ready model overtakes the baseline after a crossover point and the gap continues to grow.

\begin{table}[H]
\centering
\caption{Token-matched results at $C = 1024$, block size 64, OWT. Final values at 1,126M tokens.}
\small
\begin{tabular}{lrrrl}
\toprule
\textbf{Comparison} & \textbf{$N{=}x$ PPL} & \textbf{$D{=}x$ PPL} & \textbf{Gap} & \textbf{Crossover} \\
\midrule
$D{=}1$ vs.\ $N{=}1$ & 114.8 & 80.4 & $-$34.4 & ${\sim}424$M \\
$D{=}2$ vs.\ $N{=}2$ & 73.7 & 66.1 & $-$7.6 & ${\sim}565$M \\
$D{=}3$ vs.\ $N{=}3$ & 62.4 & 59.4 & $-$3.0 & ${\sim}835$M \\
$D{=}6$ vs.\ $N{=}6$ & 53.0 & 52.2 & $-$0.7 & ${\sim}1{,}032$M \\
\bottomrule
\end{tabular}
\end{table}

\subsection{Fine-Tuning Details}
\label{app:fine-tuning}

Any pretrained $N$-layer transformer can be converted to a $D{=}N$ context-ready model by adding a zero-initialized correction FFN and fine-tuning. The zero initialization ensures no disruption at conversion: the correction is identically zero, so the model behaves exactly as the original transformer. As fine-tuning progresses, the correction FFN learns to exploit cached context, yielding PPL improvements. At $D{=}12$, fine-tuning causes a transient $+0.11$~PPL increase before recovering; at $D{=}24$, there is no transient increase.

\begin{table}[H]
\centering
\caption{Fine-tuning pretrained roformers to context-ready (block size 256, OWT). ``Continued baseline'' is the roformer trained for the same total iterations without conversion.}
\small
\begin{tabular}{lrrrrr}
\toprule
\textbf{Conversion} & \textbf{Baseline} & \textbf{Fine-tuned} & \textbf{Cont.\ baseline} & \textbf{$\Delta$ vs.\ cont.} & \textbf{FT iters} \\
\midrule
$N{=}12$ $C{=}1408$ $\to$ $D{=}12$ & 29.92 & 26.14 & 27.20 & $-$1.06 & 200K \\
$N{=}24$ $C{=}1024$ $\to$ $D{=}24$ & 29.42 & 28.99 & --- & $-$0.43$^*$ & 18K \\
$N{=}12$ $C{=}1024$ $\to$ $D{=}12$ & 33.41 & 32.21 & --- & $-$1.20$^*$ & 50K \\
\bottomrule
\multicolumn{6}{l}{\footnotesize $^*$Gain vs.\ pre-conversion checkpoint; continued-training control not available.}
\end{tabular}
\end{table}

\subsection{Wikipedia Results}
\label{app:wikipedia}

For completeness, we include results on English Wikipedia (BPE 16k, context 256, 100K iterations).

\begin{table}[H]
\centering
\caption{Context-ready vs.\ baselines on Wikipedia. FLOPs/tok includes the prediction head ($VC$), which is identical within each width group.}
\small
\begin{tabular}{llccl}
\toprule
& \textbf{Model} & \textbf{FLOPs/tok} & \textbf{Val PPL} & \textbf{$\Delta$} \\
\midrule
\multirow{2}{*}{$C{=}768$}
& D=5 concat & 42.5M & \textbf{16.69} & \\
& Roformer $N{=}6$ & 42.5M & 17.95 & $-$1.26 \\
\midrule
\multirow{4}{*}{$C{=}446$}
& D=6 add & 15.9M & \textbf{20.40} & \\
& Roformer-hFFN $N{=}6$ & 15.9M & 21.44 & $-$1.04 \\
\cmidrule{2-5}
& D=2 concat & 7.2M & \textbf{25.48} & \\
& Roformer $N{=}3$ & 7.2M & 27.19 & $-$1.71 \\
\bottomrule
\end{tabular}
\end{table}

\subsection{Inference Timing}
\label{app:inference-timing}

Autoregressive generation speed measured on a single A100 GPU over 10{,}000 tokens with KV caching, batch size 1 (single-sequence generation).

\begin{table}[H]
\centering
\caption{Inference latency: context-ready vs.\ standard transformers (A100, 10K tokens, KV caching).}
\label{tab:inference}
\small
\begin{tabular}{lrrrr}
\toprule
\textbf{Model} & \textbf{Params} & \textbf{tok/s} & \textbf{ms/tok} & \textbf{Speedup} \\
\midrule
$D{=}1$ $C{=}2048$ & 215M & 919 & 1.09 & \multirow{2}{*}{$2.6\times$} \\
Roformer $N{=}6$ $C{=}1088$ & 155M & 351 & 2.85 & \\
\midrule
$D{=}5$ $C{=}1120$ & 157M & 349 & 2.86 & \multirow{2}{*}{$1.7\times$} \\
Roformer $N{=}12$ $C{=}768$ & 134M & 201 & 4.96 & \\
\bottomrule
\end{tabular}
\end{table}

\noindent Per-token latency is flat across sequence length in both comparisons ($<\!4\%$ growth from $T{=}100$ to $T{=}10{,}000$), confirming that KV caching amortizes attention cost to $O(T)$ per token. The context-ready models are faster despite having more parameters, because fewer sequential layers dominate inference latency on modern GPUs. In addition, fewer layers reduce total KV cache memory ($C \times D$ per token): $D{=}1$ $C{=}2048$ uses $3.2\times$ less cache than $N{=}6$ $C{=}1088$; $D{=}5$ $C{=}1120$ uses $1.6\times$ less than $N{=}12$ $C{=}768$.

\end{document}